\documentclass[10pt,journal,compsoc]{IEEEtran}

\ifCLASSOPTIONcompsoc
  \usepackage[nocompress]{cite}
\else
  \usepackage{cite}
\fi

\ifCLASSINFOpdf
  \usepackage[pdftex]{graphicx}
  \graphicspath{{../pdf/}{../jpeg/}}
  \DeclareGraphicsExtensions{.pdf,.jpeg,.png}
\else
  \usepackage[dvips]{graphicx}
  \graphicspath{{../eps/}}
  \DeclareGraphicsExtensions{.eps}
\fi

\usepackage{amsmath,amsfonts,amssymb,amsthm}
\usepackage{algorithmic}
\usepackage{array}

\ifCLASSOPTIONcompsoc
  \usepackage[caption=false,font=footnotesize,labelfont=sf,textfont=sf]{subfig}
\else
  \usepackage[caption=false,font=footnotesize]{subfig}
\fi

\usepackage{stfloats}

\usepackage{url}

\newtheorem{theorem}{Theorem}
\newtheorem{lemma}[theorem]{Lemma}
\newtheorem{corollary}[theorem]{Corollary}
\newtheorem{definition}[theorem]{Definition}

\usepackage{tikz}
\usepackage{tikz-3dplot}
\usepackage{comment}
\usepackage[utf8]{inputenc}
\usepackage{textcomp}
\DeclareUnicodeCharacter{2203}{$\exists$}
\DeclareUnicodeCharacter{211D}{\ensuremath{\mathbb{R}}}

\begin{document}

\title{NP-Hardness of Minimizing Neurons in Two-Hidden-Layer ReLU Neural Networks}

\author{Sangrock~Lee ~\IEEEmembership{}%
\IEEEcompsocitemizethanks{\IEEEcompsocthanksitem S. Lee (corresponding author) is with the FAMU--FSU College of Engineering, Florida State University, Tallahassee, FL, 32310.\protect\\
E-mail: sl22br@fsu.edu}%
}

% The paper headers
\markboth{}%
{Lee \MakeLowercase{\textit{et al.}}: NP-Hardness of Minimizing Neurons in Two-Hidden-Layer ReLU Neural Networks}

\IEEEtitleabstractindextext{%

\begin{abstract}
A fundamental question in neural network architecture optimization is whether the minimum hidden-neuron count required to approximate a target function within a prescribed tolerance can be computed efficiently. This paper resolves this question for two-hidden-layer ReLU networks under an \(L^p(\mathbb R^d,\mathbb R^m)\) approximation constraint. For every fixed \(d\ge1\), \(m\ge1\), and \(1\le p<\infty\), we prove that computing the optimum exactly is NP-hard. The result holds even when the target is represented by a rational ReLU network whose realization is nonzero, componentwise nonnegative, compactly supported, globally Lipschitz, and continuous piecewise affine. The polynomial-time reduction from \(3\)-SAT produces an architecture gap in which unsatisfiable formulas yield an optimum of zero, whereas satisfiable formulas yield an optimum of at least \(d+2\). The proof constructs compactly supported polyhedral frustum functions realized by two-hidden-layer ReLU networks and establishes the \(L^p\)-density of finite linear combinations of box-frustum functions. The results offer theoretical justification for employing heuristic approximation methods in the design of ReLU neural networks, illustrating that attaining a minimal configuration within polynomial time is computationally unachievable.
\end{abstract}

}

\maketitle

\IEEEraisesectionheading{\section{Introduction}\label{sec:introduction}}

\IEEEPARstart{O}{ptimizing} the number of neurons in deep learning architectures is critical for balancing computational efficiency and model performance. Oversized networks increase the risk of overfitting, while undersized ones may fail to capture essential data structures\cite{Geman}. Given the extensive use of feedforward neural networks in applications such as image classification \cite{AlexNet,He}, natural language processing\cite{NLP, BERT}, and scientific computing \cite{PDE1, PDE2}, there is a need for methods to design minimal yet effective models \cite{Lecun, Lottery}. Despite efforts through analytic approaches, heuristic methods, and complexity analysis, determining the optimal number of parameters remains unresolved. Our research examines the fundamental computational hurdles in optimizing ReLU network structures and provides a complexity-theoretic rationale for approximation-based methods in network design by proving a worst-case NP-hardness result for exact hidden-neuron minimization.

Extensive work has sought analytical descriptions of neural network architectures that achieve universal function approximation. Cybenko \cite{Cybenko} showed that a single-hidden-layer network with sigmoid activation can approximate any continuous function on a compact domain arbitrarily well, provided that a sufficient finite number of hidden neurons is chosen according to the desired accuracy. Hornik et al.\ \cite{Hornik} extended this result to multilayer feedforward networks, establishing universality for Borel-measurable functions. Maiorov and Pinkus \cite{Maiorov} showed that uniform approximation of any continuous function on a compact domain with a two-hidden-layer network requires a width that grows linearly with the input dimension. Their proof employs specially tailored sigmoidal activations; analytic, strictly monotone realizations of such functions are provided by Guliyev and Ismailov \cite{Guliyev}. Because these activations are not standard choices like ReLU or tanh, they must be evaluated separately, which may limit practical applicability. Gripenberg \cite{Gripenberg} demonstrated that bounded width networks can achieve universality when depth is allowed to increase. Pinkus \cite{Pinkus_MLP} surveys theoretical properties of multilayer feedforward neural networks, including universal approximation, derivative approximation, interpolation, and approximation bounds, with an emphasis on single-hidden-layer models and brief coverage of deeper architectures. Ismailov \cite{ismailov} reviews width-minimality results that depend on input dimension. In one dimension, with suitably chosen sigmoidal activations, a single hidden neuron can achieve uniform approximation, so the required width does not depend on target accuracy. For dimensions greater than one, for any fixed width, single-hidden-layer networks cannot approximate all continuous multivariate functions arbitrarily well, so the required width must increase as the accuracy requirement tightens. For two hidden layers and a computable sigmoidal activation, universality is attainable with width that grows linearly with the input dimension. Further studies \cite{Lu, Hanin, Kidger} refined bounds on necessary network width for universal approximation. Park et al. \cite{Park} established minimal neuron requirements for ReLU networks for universal approximation in \(L^p\) space when the layer depth is allowed to grow. Cai \cite{cai} generalized the findings of Park et al.\cite{Park} to arbitrary activation functions. These analytical insights have advanced understanding, yet the precise configurations of neural networks with universal approximation capability remain under active investigation.

In addressing the challenge of determining the optimal configuration of neural networks, heuristic approaches have played a vital role, as evidenced by Neural Architecture Search (NAS). Zoph et al. \cite{Zoph} introduced a scalable and transferable NAS for image classification, utilizing a controller to enhance the search process's efficiency. Efficient Neural Architecture Search (ENAS) \cite{Pham} accelerates the NAS process by sharing parameters among child models, thereby reducing the computational resources required for architecture search. Differentiable Architecture Search (DARTS) \cite{DARTS} demonstrates that NAS can be conducted through the optimization of network architectures via gradient descent within a continuous search space. Álvarez and Salzmann \cite{Alvarez} introduced a method that dynamically optimizes the number of neurons in each layer during training by leveraging a group sparsity regularizer. Lee and Lee \cite{Lee} presented NeuralScale, a framework employing iterative pruning and 'architecture descent' to optimize neuron configurations, improving accuracy in architecture-fixed settings. Kapanova et al. \cite{Kapanova} proposed an evolutionary algorithm for neural network architecture search, considering neurons, layers, weights, and activation functions. Although significant progress has been made, identifying the optimal network size remains a challenging and unresolved problem in NAS, highlighting the complexity and exploratory nature of current methodologies.

Different aspects of neural networks are examined through the lens of computational complexity, using networks with various widths and depths. Blum and Rivest \cite{Blum} demonstrate that the process of training a neural network, characterized by a single hidden layer with merely two nodes employing a step activation function and a single-node output, is classified as an NP-complete problem. This finding underscores the inherent computational difficulties that exist even in the most basic configurations of neural networks. Furthermore, this outcome has been generalized to scenarios involving the training of neural networks with ReLU activation functions \cite{Boob, Goel}. Froese and Hertrich \cite{Froese} have shown that the training of neural networks is NP-hard in fixed dimensions. Bertschinger et al. \cite{Bertschinger} established that training a neural network with a single hidden layer is $\exists \mathbb{R}$-complete. Despite these insights, the complexity aspects of structural optimization in neural network models remain largely unexplored.

This paper establishes that exact minimization of the hidden-neuron count in two-hidden-layer ReLU networks is NP-hard for every fixed input dimension \(d\ge1\), every fixed output dimension \(m\ge1\), and every \(1\le p<\infty\). The result holds even for rational target networks whose realizations are nonzero, componentwise nonnegative, compactly supported, globally Lipschitz, and continuous piecewise affine. The reduction distinguishes instances with optimum zero from instances with optimum at least \(d+2\). This result identifies a worst-case computational obstruction to exact architecture selection and provides a complexity-theoretic basis for heuristic, approximate, or restricted architecture-selection methods for ReLU neural networks.

The remainder of this paper is organized as follows. Section II fixes the notation for ReLU neural networks, hidden-layer widths, total neuron counts, and bounded polytopes. Section III constructs compactly supported polyhedral frustum functions from bounded polytopes and shows that these functions are realized by two-hidden-layer ReLU networks. Section IV proves that finite linear combinations of compactly supported box-frustum functions are dense in \(L^p(\mathbb R^d)\) and extends this approximation result to vector-valued functions. Section V defines the exact two-hidden-layer neuron minimization problem, proves an \(L^p\) degeneracy result for networks of width at most the input dimension, constructs a satisfiability-dependent target with a nonzero nonnegative baseline, and derives an architecture gap that implies NP-hardness. Section VI concludes with implications for exact architecture selection and heuristic network design.

\section{Preliminaries}
\label{sec_preliminaries}

\begin{definition}[ReLU neural network]
\label{def:NN}
Let \(\sigma:\mathbb R\to\mathbb R\) be the ReLU activation function defined by
\[
\sigma(t)=\max\{0,t\}.
\]
A feedforward ReLU neural network with input dimension \(d\), output dimension \(m\), and \(L\) hidden layers is defined by
\[
\mathbf h^{(0)}=\mathbf x,
\]
\[
\mathbf h^{(l)}
=
\sigma
\left(
W^{(l)}\mathbf h^{(l-1)}+\mathbf b^{(l)}
\right)
\]
for \(l=1,\ldots,L\), and
\[
\mathbf y
=
W^{(L+1)}\mathbf h^{(L)}+\mathbf b^{(L+1)}.
\]
Here \(\mathbf x\in\mathbb R^d\), \(\mathbf y\in\mathbb R^m\), the ReLU activation is applied componentwise, \(W^{(l)}\) is a real matrix, and \(\mathbf b^{(l)}\) is a real vector. The width of the \(l\)-th hidden layer is the number of components of \(\mathbf h^{(l)}\). The total number of hidden neurons is the sum of the widths of all hidden layers.
\end{definition}

\begin{definition}[Bounded polytope]
\label{def_bounded_polytope}
A bounded polytope in \(\mathbb R^d\) is a nonempty bounded set of the form
\[
P
=
\bigcap_{i=1}^q
\left\{
\mathbf x\in\mathbb R^d
\mid
\ell_i(\mathbf x)\ge0
\right\},
\]
where \(q\in\mathbb N\) and
\[
\ell_i(\mathbf x)
=
\mathbf a_i^\top\mathbf x+b_i
\]
for some \(\mathbf a_i\in\mathbb R^d\) and \(b_i\in\mathbb R\).
\end{definition}

\section{Formation of Frustum Functions}
\label{sec_frustum_functions}

A frustum is the portion of a solid geometry, such as a cone or pyramid, lying between two parallel cutting planes. Frustums are illustrated in Fig. \ref{fig:Frustum}. The construction below uses a ReLU network to produce a compactly supported polyhedral plateau function. The function is equal to one on a bounded polytope and vanishes outside a compact set. On the intermediate set where its values lie strictly between zero and one, the function is piecewise affine, so its graph consists of a flat plateau over the polytope together with polyhedral lateral faces. After the total violation function is introduced, this intermediate set is defined precisely as the transition region. The resulting function is called a frustum function.

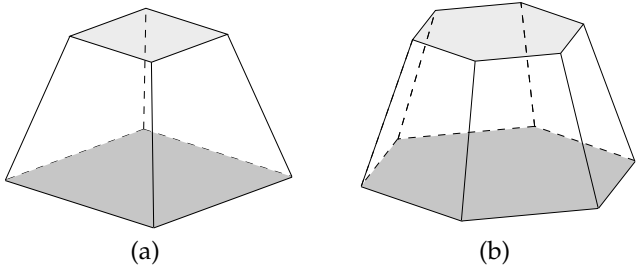
\begin{figure}[!t]
    \centering
    \begin{minipage}[b]{0.48\linewidth}
        \centering
        \resizebox{\linewidth}{!}{%
            %\documentclass{standalone}
%\usepackage{amsmath, amssymb, amsthm}
%\usepackage{tikz}
%\usepackage{tikz-3dplot}
%\usetikzlibrary{calc}
%\begin{document}
\tdplotsetmaincoords{70}{133}
\begin{tikzpicture}[tdplot_main_coords, scale=0.7]

% Define the side lengths of the squares
\def\sideBottom{4} % Side length of the bottom square
\def\sideTop{2.2} % Side length of the top square
\def\height{3} % Height of the frustum

% Bottom square vertices
\coordinate (B1) at (-\sideBottom/2, -\sideBottom/2, 0);
\coordinate (B2) at (\sideBottom/2, -\sideBottom/2, 0);
\coordinate (B3) at (\sideBottom/2, \sideBottom/2, 0);
\coordinate (B4) at (-\sideBottom/2, \sideBottom/2, 0);

% Top square vertices
\coordinate (T1) at (-\sideTop/2, -\sideTop/2, \height);
\coordinate (T2) at (\sideTop/2, -\sideTop/2, \height);
\coordinate (T3) at (\sideTop/2, \sideTop/2, \height);
\coordinate (T4) at (-\sideTop/2, \sideTop/2, \height);

% Draw the bottom square visible edges
\draw[thick] (B2) -- (B3);
\draw[thick] (B3) -- (B4);

% Draw the bottom square dashed edges
\draw[dashed] (B1) -- (B2);
\draw[dashed] (B1) -- (B4);

% Draw the top square
\draw[fill=lightgray!30] (T1) -- (T2) -- (T3) -- (T4) -- cycle;
\fill[darkgray!30] (B1) -- (B2) -- (B3) -- (B4) -- cycle;

% Draw the sides of the frustum
\draw[dashed] (B1) -- (T1);
\draw (B2) -- (T2);
\draw (B3) -- (T3);
\draw (B4) -- (T4);

\end{tikzpicture}
%\end{document}
        }
        (a)
    \end{minipage}
    \hfill
    \begin{minipage}[b]{0.48\linewidth}
        \centering
        \resizebox{\linewidth}{!}{%
            %\documentclass{standalone}
%\usepackage{amsmath, amssymb, amsthm}
%\usepackage{tikz}
%\usepackage{tikz-3dplot}
%\usetikzlibrary{calc}
%\begin{document}

\tdplotsetmaincoords{70}{105} % Setting the main coordinates to match the square frustum orientation
\begin{tikzpicture}[tdplot_main_coords]
\def\RI{2}
\def\RII{1.25}

% Draw the solid lines of the hexagon
\draw[thick] (\RI,0) -- (60:\RI) node at (60:\RI) (R1-60) {};
\draw[thick] (\RI,0) -- (0:\RI) node at (0:\RI) (R1-0) {};
\draw[thick,dashed] (120:\RI) -- (180:\RI) node at (180:\RI) (R1-180) {};
\draw[thick] (120:\RI) node at (120:\RI) (R1-120) {} -- (60:\RI);

% Draw the dashed lines of the hexagon
\draw[thick, dashed] (180:\RI) -- (240:\RI) node at (240:\RI) (R1-240) {};
\draw[thick] (300:\RI) -- (0:\RI);
\draw[thick, dashed] (240:\RI) -- (300:\RI) node at (300:\RI) (R1-300) {};

% Fill the hexagon
\path[fill=darkgray!30] (\RI,0)
  \foreach \x in {0,60,120,180,240,300} { --  (\x:\RI)};

% Draw the upper hexagon
\begin{scope}[yshift=2cm]
\draw[fill=lightgray!30] (\RII,0)
  \foreach \x in {0,60,120,180,240,300,360}
    { --  (\x:\RII) node at (\x:\RII) (R2-\x) {}};
\end{scope}

% Connect the vertices of the two hexagons
\foreach \x in {0,60,120,300} { \draw (R1-\x.center)--(R2-\x.center); };
\foreach \x in {180,240} { \draw[dashed] (R1-\x.center)--(R2-\x.center); };

% Dashed lines for the hidden edges
\draw[dashed] (R1-300.center)--(R2-300.center);
\draw[dashed] (R1-240.center)--(R2-240.center);
\draw[dashed] (R1-180.center)--(R2-180.center);
\end{tikzpicture}

%\end{document}
        }
        (b)
    \end{minipage}
    \caption{Illustration of frustums with polygonal bases. (a) Square-based frustum with parallel square top and bottom faces. (b) Hexagon-based frustum with parallel hexagonal bases connected by trapezoidal lateral faces.}
    \label{fig:Frustum}
\end{figure}

It starts from an arbitrary nonempty bounded polytope
\[
P
=
\bigcap_{i=1}^q
\left\{
\mathbf x\in\mathbb R^d
\mid
\ell_i(\mathbf x)\ge0
\right\},
\]
where
\[
\ell_i(\mathbf x)
=
\mathbf a_i^\top\mathbf x+b_i.
\]
Let \(c_i>0\) for \(i=1,\ldots,q\), and write
\[
\mathbf c=(c_1,\ldots,c_q).
\]
The first hidden layer computes the violation variables
\[
u_i(\mathbf x)
=
\sigma(-c_i\ell_i(\mathbf x)).
\]
Thus \(u_i(\mathbf x)=0\) if the \(i\)-th inequality holds, and \(u_i(\mathbf x)>0\) if the \(i\)-th inequality fails.
Define the total violation
\[
V_{P,\mathbf c}(\mathbf x)
=
\sum_{i=1}^q
\sigma(-c_i\ell_i(\mathbf x)).
\]
For \(\rho>0\), define
\[
\Theta_{P,\rho,\mathbf c}(\mathbf x)
=
\sigma
\left(
1-\frac{1}{\rho}V_{P,\mathbf c}(\mathbf x)
\right).
\]
Equivalently,
\[
\Theta_{P,\rho,\mathbf c}(\mathbf x)
=
\sigma
\left(
1-\frac{1}{\rho}
\sum_{i=1}^q
\sigma(-c_i\ell_i(\mathbf x))
\right).
\]
The transition region is
\[
\Delta_{P,\rho,\mathbf c}
=
\left\{
\mathbf x\in\mathbb R^d
\mid
0<V_{P,\mathbf c}(\mathbf x)<\rho
\right\}.
\]
On \(P\), all inequalities are satisfied, so \(V_{P,\mathbf c}=0\) and \(\Theta_{P,\rho,\mathbf c}=1\). On \(\Delta_{P,\rho,\mathbf c}\), the function is positive and piecewise affine. Whenever \(V_{P,\mathbf c}\ge\rho\), the function is zero. Hence
\[
\operatorname{supp}\Theta_{P,\rho,\mathbf c}
\subset
\left\{
\mathbf x\in\mathbb R^d
\mid
V_{P,\mathbf c}(\mathbf x)\le\rho
\right\}.
\]

This construction has a direct functional interpretation. The first hidden layer measures how much each defining inequality of \(P\) is violated. Thus all violation variables are zero on \(P\). The total violation \(V_{P,\mathbf c}\) then determines the value of the scalar function \(\Theta_{P,\rho,\mathbf c}\). Consequently, \(P\) is the unit-level plateau of \(\Theta_{P,\rho,\mathbf c}\), points with \(0<V_{P,\mathbf c}<\rho\) form the transition region, and points with \(V_{P,\mathbf c}\ge\rho\) are mapped to zero.

\begin{lemma}[Polyhedral frustum realization]
\label{lem_polyhedral_frustum_realization}
Let \(P\) be a nonempty bounded polytope in \(\mathbb R^d\) of the form
\[
P
=
\bigcap_{i=1}^q
\left\{
\mathbf x\in\mathbb R^d
\mid
\ell_i(\mathbf x)\ge0
\right\},
\]
where
\[
\ell_i(\mathbf x)=\mathbf a_i^\top\mathbf x+b_i.
\]
Let \(c_i>0\) for \(i=1,\ldots,q\), let \(\mathbf c=(c_1,\ldots,c_q)\), and let \(\rho>0\). Then
\[
0\le\Theta_{P,\rho,\mathbf c}\le1,
\]
\[
\Theta_{P,\rho,\mathbf c}=1
\]
on \(P\), and \(\Theta_{P,\rho,\mathbf c}\) is compactly supported. Moreover, \(\Theta_{P,\rho,\mathbf c}\) is realized by a two-hidden-layer ReLU network with \(q\) neurons in the first hidden layer and one neuron in the second hidden layer.
\end{lemma}

\begin{proof}
For every \(i\), the quantity \(\sigma(-c_i\ell_i(\mathbf x))\) is nonnegative. Hence
\[
V_{P,\mathbf c}(\mathbf x)\ge0,
\]
and therefore
\[
1-\frac{1}{\rho}V_{P,\mathbf c}(\mathbf x)\le1.
\]
Since ReLU is nonnegative and satisfies \(\sigma(t)\le1\) whenever \(t\le1\), it follows that
\[
0\le\Theta_{P,\rho,\mathbf c}(\mathbf x)\le1
\]
for every \(\mathbf x\in\mathbb R^d\).

If \(\mathbf x\in P\), then \(\ell_i(\mathbf x)\ge0\) for every \(i\). Since \(c_i>0\), we have
\[
\sigma(-c_i\ell_i(\mathbf x))=0
\]
for every \(i\). Hence
\[
V_{P,\mathbf c}(\mathbf x)=0,
\]
and so
\[
\Theta_{P,\rho,\mathbf c}(\mathbf x)=1.
\]

The network realization follows from the displayed formula. The first hidden layer computes the \(q\) functions
\[
\sigma(-c_i\ell_i(\mathbf x))
\]
for \(i=1,\ldots,q\). The second hidden layer computes
\[
\sigma
\left(
1-\frac{1}{\rho}
\sum_{i=1}^q
\sigma(-c_i\ell_i(\mathbf x))
\right).
\]
The affine output layer is the identity.

It remains to prove compact support. Choose \(\mathbf x_0\in P\). Since \(P\) is bounded, it cannot contain a ray starting from \(\mathbf x_0\). Hence there is no nonzero vector \(\mathbf v\) such that
\[
\mathbf a_i^\top\mathbf v\ge0
\]
for every \(i=1,\ldots,q\). For contradiction, assume that such a nonzero vector \(\mathbf v\) exists, then for every \(t\ge0\) and every \(i=1,\ldots,q\),
\[
\ell_i(\mathbf x_0+t\mathbf v)
=
\ell_i(\mathbf x_0)+t\mathbf a_i^\top\mathbf v
\ge0.
\]
Thus the point \(\mathbf x_0+t\mathbf v\) satisfies every defining inequality of \(P\). Hence
\[
\{\mathbf x_0+t\mathbf v\mid t\ge0\}\subset P.
\]
Since \(\mathbf v\ne0\),
\[
\|\mathbf x_0+t\mathbf v\|_2
\ge
t\|\mathbf v\|_2-\|\mathbf x_0\|_2,
\]
and the right-hand side becomes arbitrarily large as \(t\to\infty\). Therefore \(P\) would be unbounded, contradicting the boundedness of \(P\). 
Let
\[
S
=
\left\{
\mathbf v\in\mathbb R^d
\mid
\|\mathbf v\|_2=1
\right\}.
\]
Fix \(\mathbf v\in S\). Then \(\mathbf v\ne0\). If
\[
\mathbf a_i^\top\mathbf v\ge0
\]
for all \(i=1,\ldots,q\), then \(\mathbf v\) would be a nonzero vector with the property ruled out above. Hence these inequalities cannot all hold for this chosen direction \(\mathbf v\). Therefore there exists an index \(i=i(\mathbf v)\) such that
\[
\mathbf a_i^\top\mathbf v<0.
\]
Define
\[
\gamma(\mathbf v)
=
\sum_{j=1}^q
\sigma(-c_j\mathbf a_j^\top\mathbf v).
\]
For the index \(i=i(\mathbf v)\) found above, since \(c_i>0\), we have
\[
-c_i\mathbf a_i^\top\mathbf v>0.
\]
Thus
\[
\sigma(-c_i\mathbf a_i^\top\mathbf v)>0.
\]
Each summand in \(\gamma(\mathbf v)\) is a ReLU output, so no summand is negative. Hence this positive summand cannot be canceled, and therefore
\[
\gamma(\mathbf v)>0.
\]
Since the choice of \(\mathbf v\in S\) was arbitrary, \(\gamma\) is positive on \(S\).

The function \(\gamma\) is continuous because it is a finite sum of continuous functions. Since \(S\) is compact, the minimum
\[
\gamma_0
=
\min_{\mathbf v\in S}
\gamma(\mathbf v)
\]
exists. Since \(\gamma\) is positive on \(S\), it follows that \(\gamma_0>0\). Therefore
\[
\gamma(\mathbf v)\ge\gamma_0
\]
for \(\mathbf v\in S\).
Let
\[
B
=
\sum_{i=1}^q
c_i|b_i|.
\]
For \(\mathbf x\ne0\), write \(\mathbf x=t\mathbf v\), where \(t=\|\mathbf x\|_2\) and \(\|\mathbf v\|_2=1\). Using the inequality \(\sigma(r-s)\ge\sigma(r)-|s|\), we obtain
\[
\begin{aligned}
V_{P,\mathbf c}(t\mathbf v)
&=
\sum_{i=1}^q
\sigma(-c_i(t\mathbf a_i^\top\mathbf v+b_i))
\\
&=
\sum_{i=1}^q
\sigma(-tc_i\mathbf a_i^\top\mathbf v-c_ib_i)
\\
&\ge
\sum_{i=1}^q
\sigma(-tc_i\mathbf a_i^\top\mathbf v)
-
\sum_{i=1}^q
c_i|b_i|
\\
&=
t
\sum_{i=1}^q
\sigma(-c_i\mathbf a_i^\top\mathbf v)
-
B
\\
&\ge
t\gamma_0-B.
\end{aligned}
\]
If
\[
t>\frac{\rho+B}{\gamma_0},
\]
then
\[
V_{P,\mathbf c}(t\mathbf v)>\rho.
\]
For such \(t\), the input to the outer ReLU in the definition of \(\Theta_{P,\rho,\mathbf c}\) is negative. Hence
\[
\Theta_{P,\rho,\mathbf c}(t\mathbf v)=0.
\]
Thus \(\Theta_{P,\rho,\mathbf c}\) vanishes outside the Euclidean ball
\[
\left\{
\mathbf x\in\mathbb R^d
\mid
\|\mathbf x\|_2\le
\frac{\rho+B}{\gamma_0}
\right\}.
\]
Therefore \(\Theta_{P,\rho,\mathbf c}\) is compactly supported.
\end{proof}

Fig. \ref{fig:Transition region} illustrates the transition region. The central set labeled \(P\) is the bounded polytope, where \(V_{P,\mathbf c}=0\) and \(\Theta_{P,\rho,\mathbf c}=1\). The surrounding set labeled \(\Delta_{P,\rho,c}\) represents the transition region where \(0<V_{P,\mathbf c}<\rho\).

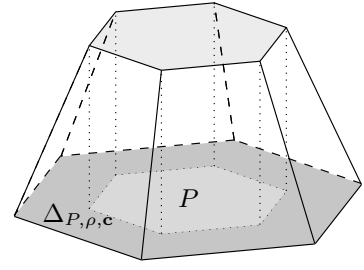
\begin{figure}
    \centering
    \resizebox{0.6\linewidth}{!}{%
        %\documentclass{standalone}
%\usepackage{amsmath, amssymb, amsthm}
%\usepackage{tikz}
%\usepackage{tikz-3dplot}
%\usetikzlibrary{calc}
%\begin{document}
\tdplotsetmaincoords{70}{105} % Setting the main coordinates to match the square frustum orientation
\begin{tikzpicture}[tdplot_main_coords]
\def\RI{2.2}
\def\RII{1.25}
\def\RIII{1.25}

\node (R3-0) at (\RIII,0) {};
\node (R3-60) at (60:\RIII) {};
\node (R3-120) at (120:\RIII) {};
\node (R3-180) at (180:\RIII) {};
\node (R3-240) at (240:\RIII) {};
\node (R3-300) at (300:\RIII) {};

%\draw[dotted, line width = 0.4] (R3-0)--(R3-60)--(R3-120)--(R3-180)--(R3-240)--(R3-300)--(R3-0);

% Draw the solid lines of the hexagon
\draw[thick] (\RI,0) -- (60:\RI) node at (60:\RI) (R1-60) {};
\draw[thick] (\RI,0) -- (0:\RI) node at (0:\RI) (R1-0) {};
\draw[thick,dashed] (120:\RI) -- (180:\RI) node at (180:\RI) (R1-180) {};
\draw[thick] (120:\RI) node at (120:\RI) (R1-120) {} -- (60:\RI);

% Draw the dashed lines of the hexagon
\draw[thick, dashed] (180:\RI) -- (240:\RI) node at (240:\RI) (R1-240) {};
\draw[thick] (300:\RI) -- (0:\RI);
\draw[thick, dashed] (240:\RI) -- (300:\RI) node at (300:\RI) (R1-300) {};

% Fill the hexagon
\path[fill=darkgray!30] (\RI,0)
  \foreach \x in {0,60,120,180,240,300} { --  (\x:\RI)};

% Draw the upper hexagon
\begin{scope}[yshift=2cm]
\draw[fill=lightgray!30] (\RII,0)
  \foreach \x in {0,60,120,180,240,300,360}
    { --  (\x:\RII) node at (\x:\RII) (R2-\x) {}};
\end{scope}

% Dashed lines for the hidden edges
\draw[dashed] (R1-300.center)--(R2-300.center);
\draw[dashed] (R1-240.center)--(R2-240.center);
\draw[dashed] (R1-180.center)--(R2-180.center);

\draw[dotted,black] (\RIII,0)
 \foreach \x in {0,60,120,180,240,300,360}
    { --  (\x:\RIII) node at (\x:\RIII) (R3-\x) {}};

\path[fill=gray!30] (\RIII,0)
  \foreach \x in {0,60,120,180,240,300} { --  (\x:\RIII)};

\draw[dotted] (R3-0.center)--(R2-0.center);
\draw[dotted] (R3-60.center)--(R2-60.center);
\draw[dotted] (R3-120.center)--(R2-120.center);
\draw[dotted] (R3-180.center)--(R2-180.center);
\draw[dotted] (R3-240.center)--(R2-240.center);
\draw[dotted] (R3-300.center)--(R2-300.center);

% Connect the vertices of the two hexagons
\foreach \x in {0,60,120,300} { \draw (R1-\x.center)--(R2-\x.center); };
\foreach \x in {180,240} { \draw[dashed] (R1-\x.center)--(R2-\x.center); };

\node (P) at (0.9,0.25) [label={\small $P$}] {};
\node (Delta) at (2, -0.85) [label={\small $\Delta_{P,\rho, \mathbf c}$}] {};

\end{tikzpicture}

%\end{document}
    }
    \caption{Transition region of a frustum function. The central set labeled \(P\) is the bounded polytope, where the function equals one. The surrounding set labeled \(\Delta_{P,\rho,\textbf{c}}\) represents the transition region, where \(0<V_{P,\mathbf c}<\rho\) and the function decreases piecewise affinely from one to zero}
    \label{fig:Transition region}
\end{figure}

The constants \(c_i\) control the slopes of the lateral faces. Multiplying an inequality by \(c_i>0\) does not change the polytope, because
\[
c_i\ell_i(\mathbf x)\ge0
\]
is equivalent to
\[
\ell_i(\mathbf x)\ge0.
\]
However, it changes how fast the frustum function decreases after the corresponding inequality is violated.

Let \(I\subset\{1,\ldots,q\}\) be an active violation set. On a region where
\[
\ell_i(\mathbf x)<0
\]
for \(i\in I\),
\[
\ell_i(\mathbf x)>0
\]
for \(i\notin I\), and
\[
0<V_{P,\mathbf c}(\mathbf x)<\rho,
\]
the outer ReLU is active and the active set is fixed. On that region,
\[
\Theta_{P,\rho,\mathbf c}(\mathbf x)
=
1+\frac{1}{\rho}
\sum_{i\in I}
c_i\ell_i(\mathbf x).
\]
Therefore
\[
\nabla\Theta_{P,\rho,\mathbf c}(\mathbf x)
=
\frac{1}{\rho}
\sum_{i\in I}
c_i\mathbf a_i.
\]
In particular, if only the \(i\)-th inequality is violated, then
\[
\nabla\Theta_{P,\rho,\mathbf c}(\mathbf x)
=
\frac{c_i}{\rho}\mathbf a_i.
\]
Thus increasing \(c_i\) increases the magnitude of the corresponding transition slope, and decreasing \(c_i\) decreases it. At points where the active violation set changes or where \(V_{P,\mathbf c}=\rho\), the function may have a kink. This behavior is illustrated in Fig. \ref{fig:Scale factor}.

\begin{figure}
    \centering
    \resizebox{0.8\linewidth}{!}{%
        %\documentclass{standalone}
%\usepackage{amsmath, amssymb, amsthm}
%\usepackage{tikz}
%\usepackage{tikz-3dplot}
%\usetikzlibrary{calc}
\usetikzlibrary{decorations.pathreplacing}
%\begin{document}
\begin{tikzpicture}
\def\height{3.5}

\draw[thick,->] (-0.2,0) -- (5,0) node[right] {$x$};
\draw[thick,->] (0,-0.2) -- (0,4.5) node[right] {$y$};
%\draw[step=1cm,gray,very thin] (-0.2,-0.2) grid (4.9,4.9);

\draw [thick] (0,0)--(1,0)--(2,\height)--(3,\height)--(4,0)--(5,0);
\draw [dashed] (3,\height)--(3.8,0);
\draw [dashed] (3,\height)--(3.6,0);
\draw [dashed] (3,\height)--(3.4,0);

\draw [dashed] (0.6,0)--(2,\height);
\draw [dashed] (0.8,0)--(2,\height);
\draw [dashed] (0.4,0)--(2,\height);

\draw [dotted] (2,\height)--(2,0);
\draw [dotted] (3,\height)--(3,0);

%\node at (2.5,-0.3) {$E_1$};

\draw [thick,decorate, decoration={brace, amplitude=5pt, mirror, raise=0.4ex}]
  (2,0) -- (3,0) node[midway, yshift=-1.4em] {$E_1$};

\draw[thick,->] (1.17,2.2)--+(0,0.3) node[left] {};
\node at (0.85,2.35) {$c_{L}$};

\draw[thick,->] (3.97,2.48)--+(0,-0.3) {};
\node at (3.65,2.35) {$c_{R}$};

\draw[-latex,red] (0.7,1) arc [start angle=270, end angle=300, radius=1.7];
\draw[-latex,blue] (3.2,1.26) arc [start angle=230, end angle=268, radius=1.3];

\end{tikzpicture}

%\end{document}
    }
    \caption{Effect of scaling factors on a one-dimensional frustum function. The plateau is the interval \(E_1\). The sloped parts are transition regions. Increasing \(c_{L}\) increases the magnitude of the left transition slope, while decreasing \(c_{R}\) decreases the magnitude of the right transition slope.}
    \label{fig:Scale factor}
\end{figure}
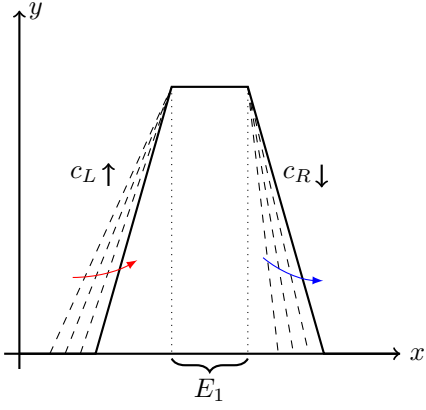

When
\[
P=
\prod_{j=1}^d
[\alpha_j,\beta_j]
\]
is an axis-aligned bounded box and all scaling constants are equal to one, the frustum function takes the form
\[
\Theta_{P,\rho}(\mathbf x)
=
\sigma
\left(
1-\frac{1}{\rho}
\sum_{j=1}^d
\left(
\sigma(\alpha_j-x_j)
+
\sigma(x_j-\beta_j)
\right)
\right).
\]
This function equals one on \(P\) and vanishes outside
\[
P^\rho
=
\prod_{j=1}^d
[\alpha_j-\rho,\beta_j+\rho].
\]
Such a function is called a box-frustum function.

\section{Function Approximation by Frustum Functions}

The previous section showed that frustum functions associated with nonempty bounded polytopes are realized by two-hidden-layer ReLU networks and are compactly supported. For the approximation argument, it is enough to use the special case where the polytope is an axis-aligned bounded box. This section proves that finite linear combinations of compactly supported box-frustum functions are dense in \(L^p(\mathbb R^d)\). The proof uses the standard density of finite linear combinations of characteristic functions of bounded boxes and replaces each box indicator by a box-frustum function with a sufficiently thin transition region.

\begin{lemma}[Frustum density]
\label{lem_frustum_density}
Let \(d\ge1\), let \(1\le p<\infty\), and let \(f\in L^p(\mathbb R^d)\). For every \(\varepsilon>0\), there exists a scalar-valued two-hidden-layer ReLU network \(\phi\) from \(\mathbb R^d\) to \(\mathbb R\) such that \(\phi\in L^p(\mathbb R^d)\) and
\[
\|f-\phi\|_{L^p(\mathbb R^d)}<\varepsilon.
\]
Moreover, \(\phi\) can be chosen as a finite linear combination of compactly supported box-frustum functions.
\end{lemma}

\begin{proof}
Since finite linear combinations of characteristic functions of bounded boxes are dense in \(L^p(\mathbb R^d)\), there exist bounded boxes \(Q_1,\ldots,Q_N\subset\mathbb R^d\) and constants \(c_1,\ldots,c_N\in\mathbb R\) such that
\[
g=\sum_{k=1}^N c_k\chi_{Q_k}
\]
satisfies
\[
\|f-g\|_{L^p(\mathbb R^d)}<\frac{\varepsilon}{2}.
\]
Write
\[
Q_k=\prod_{i=1}^d[a_{ki},b_{ki}].
\]
For \(\rho>0\), define
\[
\theta_{k,\rho}(\mathbf x)
=
\sigma
\left(
1-\frac{1}{\rho}
\sum_{i=1}^d
\left(
\sigma(a_{ki}-x_i)+\sigma(x_i-b_{ki})
\right)
\right).
\]
By Lemma \ref{lem_polyhedral_frustum_realization}, each \(\theta_{k,\rho}\) is a compactly supported two-hidden-layer ReLU box-frustum function. Moreover,
\[
0\le\theta_{k,\rho}\le1
\]
on \(\mathbb R^d\),
\[
\theta_{k,\rho}=1
\]
on \(Q_k\), and
\[
\theta_{k,\rho}=0
\]
outside
\[
Q_k^\rho
=
\prod_{i=1}^d[a_{ki}-\rho,b_{ki}+\rho].
\]
Define
\[
\phi_\rho=\sum_{k=1}^N c_k\theta_{k,\rho}.
\]
This function is realized by one two-hidden-layer ReLU network by placing the subnetworks for \(\theta_{1,\rho},\ldots,\theta_{N,\rho}\) in parallel and combining their outputs in the affine output layer.

Since
\[
\mu(Q_k^\rho\setminus Q_k)\to0
\]
as \(\rho\to0\), choose \(\rho>0\) such that
\[
N^{p-1}
\sum_{k=1}^N
|c_k|^p
\mu(Q_k^\rho\setminus Q_k)
<
\left(\frac{\varepsilon}{2}\right)^p.
\]
Since \(\theta_{k,\rho}-\chi_{Q_k}\) vanishes outside \(Q_k^\rho\setminus Q_k\) and has absolute value at most \(1\), the discrete Holder inequality gives
\[
|\phi_\rho(\mathbf x)-g(\mathbf x)|^p
\le
N^{p-1}
\sum_{k=1}^N
|c_k|^p
|\theta_{k,\rho}(\mathbf x)-\chi_{Q_k}(\mathbf x)|^p.
\]
Therefore
\[
\|\phi_\rho-g\|_{L^p(\mathbb R^d)}^p
\le
N^{p-1}
\sum_{k=1}^N
|c_k|^p
\mu(Q_k^\rho\setminus Q_k)
<
\left(\frac{\varepsilon}{2}\right)^p.
\]
Hence
\[
\|\phi_\rho-g\|_{L^p(\mathbb R^d)}<\frac{\varepsilon}{2}.
\]
By the triangle inequality,
\[
\|f-\phi_\rho\|_{L^p(\mathbb R^d)}
\le
\|f-g\|_{L^p(\mathbb R^d)}
+
\|g-\phi_\rho\|_{L^p(\mathbb R^d)}
<
\varepsilon.
\]
Setting \(\phi=\phi_\rho\) proves the claim.
\end{proof}

\begin{corollary}[Vector-valued frustum density]
\label{cor_vector_frustum_density}
Let \(d\ge1\), let \(m\ge1\), let \(1\le p<\infty\), and let \(\mathbf f=(f_1,\ldots,f_m)\in L^p(\mathbb R^d,\mathbb R^m)\), where
\[
\|\mathbf f\|_{L^p(\mathbb R^d,\mathbb R^m)}
=
\left(
\int_{\mathbb R^d}
\sum_{j=1}^m |f_j(\mathbf x)|^p
\mathrm d\mathbf x
\right)^{1/p}.
\]
For every \(\varepsilon>0\), there exists a two-hidden-layer ReLU network \(\Phi\) from \(\mathbb R^d\) to \(\mathbb R^m\) such that
\[
\|\mathbf f-\Phi\|_{L^p(\mathbb R^d,\mathbb R^m)}
<
\varepsilon.
\]
Each coordinate of \(\Phi\) can be chosen as a finite linear combination of compactly supported box-frustum functions.
\end{corollary}

\begin{proof}
For each \(j=1,\ldots,m\), Lemma \ref{lem_frustum_density} gives a scalar-valued two-hidden-layer ReLU network \(\phi_j\) such that
\[
\|f_j-\phi_j\|_{L^p(\mathbb R^d)}
<
\varepsilon m^{-1/p}.
\]
Place these scalar networks in parallel and let the affine output layer return
\[
\Phi=(\phi_1,\ldots,\phi_m).
\]
Then
\[
\|\mathbf f-\Phi\|_{L^p(\mathbb R^d,\mathbb R^m)}^p
=
\sum_{j=1}^m
\|f_j-\phi_j\|_{L^p(\mathbb R^d)}^p
<
\sum_{j=1}^m
\varepsilon^p m^{-1}
=
\varepsilon^p.
\]
Taking the \(p\)-th root gives
\[
\|\mathbf f-\Phi\|_{L^p(\mathbb R^d,\mathbb R^m)}
<
\varepsilon.
\]
\end{proof}

\section{Computational Complexity of Minimizing Neuron Numbers in ReLU Networks}
\label{sec_computational_complexity}

\begin{definition}[Two-hidden-layer neuron minimization]
\label{def_neuron_minimization}
Fix integers \(d\ge1\) and \(m\ge1\), and fix \(1\le p<\infty\). For a measurable function
\[
\mathbf g=(g_1,\ldots,g_m):\mathbb R^d\to\mathbb R^m,
\]
define
\[
\|\mathbf g\|_{L^p(\mathbb R^d,\mathbb R^m)}
=
\left(
\int_{\mathbb R^d}
\sum_{j=1}^m
|g_j(\mathbf x)|^p
\mathrm d\mathbf x
\right)^{1/p}.
\]

Let \(\mathfrak N_{N,d,m}^{(2)}\) be the class of functions from \(\mathbb R^d\) to \(\mathbb R^m\) realized by ReLU networks with two hidden layers and at most \(N\) hidden neurons in total. The parameters of an approximating network may be arbitrary real numbers. A network having a zero-width hidden layer is interpreted as a constant function.

For a rational ReLU target network
\[
\mathcal S:\mathbb R^d\to\mathbb R^m
\]
whose realization \(\mathbf f_{\mathcal S}\) belongs to \(L^p(\mathbb R^d,\mathbb R^m)\), and for a positive rational number \(\varepsilon\), define
\[
\begin{aligned}
\operatorname{OPT}_{p,d,m}(\mathcal S,\varepsilon)
&=
\min
\left\{
N\in\mathbb N_0
\middle|
\exists\Phi\in\mathfrak N_{N,d,m}^{(2)}
\right.
\\
&\left.
\text{ such that }
\|\Phi-\mathbf f_{\mathcal S}\|_{L^p(\mathbb R^d,\mathbb R^m)}
<
\varepsilon
\right\}.
\end{aligned}
\]
The minimum of the empty set is defined as \(+\infty\).
\end{definition}

\begin{lemma}[Width-\(d\) \(L^p\) degeneracy]
\label{lem_width_d_Lp_degeneracy}
Let \(d\ge1\), \(m\ge1\) and $1\leq p< \infty $. Let
\[
\Phi:\mathbb R^d\to\mathbb R^m
\]
be realized by a ReLU network in which every hidden layer has width at most \(d\). If
\[
\Phi\in L^p(\mathbb R^d,\mathbb R^m),
\]
then
\[
\Phi=0.
\]
\end{lemma}

\begin{proof}
For every output coordinate \(j\), the scalar function \(\Phi_j\) is realized by a ReLU network whose hidden layers have width at most \(d\). The scalar width-\(d\) degeneracy result states that such a network is either identically zero or does not belong to \(L^p(\mathbb R^d)\) \cite[Appendix B.3]{Park}. Since
\[
\Phi_j\in L^p(\mathbb R^d),
\]
it follows that
\[
\Phi_j=0
\]
for every \(j=1,\ldots,m\). Therefore
\[
\Phi=0.
\]
\end{proof}

\begin{lemma}[Minimum nonzero \(L^p\) neuron count]
\label{lem_minimum_nonzero_Lp_neurons}
Let \(d\ge1\), \(m\ge1\) and $1\leq p< \infty $. Let
\[
\Phi:\mathbb R^d\to\mathbb R^m
\]
be a nonzero function realized by a two-hidden-layer ReLU network. If
\[
\Phi\in L^p(\mathbb R^d,\mathbb R^m),
\]
then the network has at least \(d+2\) hidden neurons in total.
\end{lemma}

\begin{proof}
Let \(n_1\) and \(n_2\) denote the widths of the two hidden layers. If either width is zero, then \(\Phi\) is constant. Since the only constant function in \(L^p(\mathbb R^d,\mathbb R^m)\) is the zero function, this contradicts the assumption that \(\Phi\ne0\). Hence
\[
n_1\ge1
\]
and
\[
n_2\ge1.
\]

Suppose that
\[
n_1+n_2\le d+1.
\]
Since both widths are positive,
\[
n_1\le d
\]
and
\[
n_2\le d.
\]
Thus, every hidden layer has width at most \(d\). Lemma \ref{lem_width_d_Lp_degeneracy} gives
\[
\Phi=0,
\]
which is a contradiction. Thus
\[
n_1+n_2\ge d+2.
\]
\end{proof}

\begin{lemma}[Frustum function satisfiability target]
\label{lem_sat_target}
Fix \(d\ge1\) and $1\leq p < \infty$. Given a \(3\)-CNF formula \(\Psi\) with \(n\ge1\) variables and \(s\ge1\) clauses, one can construct in polynomial time a rational ReLU network whose realization
\[
f_{\Psi,d}:\mathbb R^d\to\mathbb R
\]
is nonnegative, compactly supported, globally Lipschitz, and continuous piecewise affine. Its support satisfies
\[
\operatorname{supp}f_{\Psi,d}
\subset
[0,1]^d.
\]
With
\[
\eta=2^{-(n+4)}
\]
and
\[
\nu_{\eta,d}
=
2\eta
\left(\frac12\right)^{d-1},
\]
the construction satisfies
\[
\Psi\notin3\text{-SAT}
\Longrightarrow
f_{\Psi,d}=0
\]
and
\[
\Psi\in3\text{-SAT}
\Longrightarrow
\|f_{\Psi,d}\|_{L^p(\mathbb R^d)}
\ge
\nu_{\eta,d}^{1/p}.
\]
If \(\Psi\) is satisfiable, there exists a measurable box of measure \(\nu_{\eta,d}\) on which
\[
f_{\Psi,d}=1.
\]
\end{lemma}

\begin{proof}
Write the variables of \(\Psi\) as \(X_1,\ldots,X_n\) and its clauses as \(C_1,\ldots,C_s\). Define
\[
B_\eta(t)
=
\frac{
\sigma\left(t-\frac12+\eta\right)
-
\sigma\left(t-\frac12-\eta\right)
}{2\eta}.
\]
For every \(t\in\mathbb R\),
\[
B_\eta(t)
=
\begin{cases}
0 & t\le\frac12-\eta,\\
\displaystyle
\frac{t-\frac12+\eta}{2\eta}
& \frac12-\eta\le t\le\frac12+\eta,\\
1 & t\ge\frac12+\eta.
\end{cases}
\]
Thus
\[
0\le B_\eta(t)\le1.
\]

For \(x\in\mathbb R\), define
\[
r_0(x)=\sigma(x).
\]
For \(i=1,\ldots,n\), define recursively
\[
z_i(x)=B_\eta(r_{i-1}(x))
\]
and
\[
r_i(x)
=
\sigma
\left(
2r_{i-1}(x)-z_i(x)
\right).
\]

For \(x\in[0,1]\), one has
\[
r_0(x)=x\in[0,1].
\]
We now show that if \(r_{i-1}(x)\in[0,1]\), then \(r_i(x)\in[0,1]\). Indeed, suppose that \(r\in[0,1]\). If
\[
r\le\frac12-\eta,
\]
then \(B_\eta(r)=0\) and
\[
0\le2r\le1-2\eta.
\]
If
\[
r\ge\frac12+\eta,
\]
then \(B_\eta(r)=1\) and
\[
0\le2r-1\le1.
\]
On the transition interval, the function
\[
2r-B_\eta(r)
\]
is affine with slope
\[
2-\frac{1}{2\eta}<0.
\]
Its maximum on that interval is attained at \(r=\frac12-\eta\) and equals \(1-2\eta\). Applying ReLU therefore produces a value in \([0,1]\). Consequently,
\[
r_i(x)\in[0,1]
\]
and
\[
z_i(x)\in[0,1]
\]
for every \(x\in[0,1]\).

For \(\textbf{z}=(z_1,\ldots,z_n)\in[0,1]^n\), assign a numerical value to each literal. The value of \(X_i\) is \(z_i\), and the value of its negation \(\neg X_i\) is \(1-z_i\). Thus,
\[
\ell_{X_i}(\textbf{z})=z_i
\]
and
\[
\ell_{\neg X_i}(\textbf{z})=1-z_i.
\]
If the literals of \(C_j\) are
\[
\lambda_{j1},\lambda_{j2},\lambda_{j3},
\]
define
\[
d_j(\textbf{z})
=
\sigma
\left(
1
-
\ell_{\lambda_{j1}}(\textbf{z})
-
\ell_{\lambda_{j2}}(\textbf{z})
-
\ell_{\lambda_{j3}}(\textbf{z})
\right).
\]
For a Boolean vector, \(d_j\) equals one exactly when \(C_j\) is false.

Define
\[
P(\textbf{z})
=
\sum_{i=1}^n
\left(
z_i-\sigma(2z_i-1)
\right).
\]
Choose a nearest Boolean vector
\[
\bar{\mathbf{z}}\in\{0,1\}^n
\]
coordinatewise. For every \(i\),
\[
z_i-\sigma(2z_i-1)
=
|z_i-\bar z_i|.
\]
Hence
\[
P(\textbf{z})
=
\sum_{i=1}^n
|z_i-\bar z_i|.
\]
Define
\[
S_\Psi(\textbf{z})
=
1
-
\sum_{j=1}^s d_j(\textbf{z})
-
(3s+1)P(\textbf{z})
\]
and
\[
G_\Psi(x)
=
\sigma
\left(
S_\Psi
\left(
z_1(x),\ldots,z_n(x)
\right)
\right).
\]

Since ReLU is \(1\)-Lipschitz,
\[
|d_j(\textbf{z})-d_j(\bar{\mathbf{z}})|
\le
\sum_{k=1}^3
\left|
\ell_{\lambda_{jk}}(\textbf{z})
-
\ell_{\lambda_{jk}}(\bar{\mathbf{z}})
\right|.
\]
Each literal is either \(X_i\) or \(\neg X_i\).

For \(X_i\),
\[
\left|
\ell_{X_i}(\mathbf z)
-
\ell_{X_i}(\bar{\mathbf z})
\right|
=
|z_i-\bar z_i|.
\]

For \(\neg X_i\),
\[
\begin{aligned}
\left|
\ell_{\neg X_i}(\mathbf z)
-
\ell_{\neg X_i}(\bar{\mathbf z})
\right|
&=
|(1-z_i)-(1-\bar z_i)|
\\
&=
|z_i-\bar z_i|.
\end{aligned}
\]

Since
\[
P(\mathbf z)
=
\sum_{i=1}^n |z_i-\bar z_i|,
\]
each literal difference is at most \(P(\mathbf z)\). A clause has three literals, so the sum of its three literal differences is at most
\[
3P(\mathbf z).
\]
Therefore,
\[
|d_j(\mathbf z)-d_j(\bar{\mathbf z})|
\le
3P(\mathbf z).
\]

Suppose that \(\Psi\) is unsatisfiable. Every Boolean vector falsifies at least one clause, so
\[
\sum_{j=1}^s d_j(\bar{\mathbf{z}})\ge1.
\]
It follows that
\[
\sum_{j=1}^s d_j(\textbf{z})
\ge
\sum_{j=1}^s d_j(\bar{\mathbf{z}})-3sP(\textbf{z})
\ge
1-3sP(\textbf{z}).
\]
Consequently,
\[
S_\Psi(\textbf{z})
\le
-P(\textbf{z})
\le0.
\]
Thus
\[
G_\Psi(x)=0
\]
for every \(x\in[0,1]\).

Suppose that \(\Psi\) is satisfiable. Let
\[
\textbf{a}=(a_1,\ldots,a_n)\in\{0,1\}^n
\]
be a satisfying assignment and define
\[
x_\textbf{a}
=
\sum_{i=1}^n
a_i2^{-i}
+
2^{-(n+1)}.
\]
For \(i=1,\ldots,n\), define
\[
\Delta_i
=
2^{-(n-i+2)}.
\]
After the first \(i-1\) digits have been extracted, the remainder at \(x_a\) is
\[
r_{i-1}(x_\textbf{a})
=
\sum_{k=i}^n
a_k2^{-(k-i+1)}
+
2^{-(n-i+2)}.
\]
If \(a_i=1\), then
\[
r_{i-1}(x_\textbf{a})
\ge
\frac12+\Delta_i.
\]
If \(a_i=0\), then
\[
r_{i-1}(x_\textbf{a})
\le
\frac12-\Delta_i.
\]
Since
\[
\Delta_i>\eta,
\]
the remainder lies outside the transition interval of \(B_\eta\). Hence
\[
z_i(x_\textbf{a})=a_i.
\]
The update then removes the leading binary digit and gives the displayed remainder formula for the next index.

Let
\[
I_\textbf{a}=[x_\textbf{a}-\eta,x_\textbf{a}+\eta].
\]
For any \(x\in I_\textbf{a}\),
\[
|r_0(x)-r_0(x_\textbf{a})|
\le
\eta.
\]
We have shown that \(z_i(x_\textbf{a})=a_i\), and we now prove by induction on \(i\) that
\[
z_i(x)=z_i(x_\textbf{a})=a_i
\]
for every \(x\in I_\textbf{a}\) and every \(i=1,\ldots,n\). When \(i=1\), there is no index \(t<i\), so the induction hypothesis is vacuously true. Suppose that
\[
z_t(x)=z_t(x_\textbf{a})=a_t
\]
for every \(t<i\).
Since ReLU is \(1\)-Lipschitz,
\[
|r_t(x)-r_t(x_\textbf{a})|
\le
2|r_{t-1}(x)-r_{t-1}(x_\textbf{a})|.
\]
Therefore
\[
|r_{i-1}(x)-r_{i-1}(x_\textbf{a})|
\le
2^{i-1}\eta.
\]
Moreover,
\[
2^{i-1}\eta+\eta
\le
2^i\eta
=
\frac14\Delta_i
<
\Delta_i.
\]
If \(a_i=1\), then
\[
r_{i-1}(x)
>
\frac12+\eta.
\]
If \(a_i=0\), then
\[
r_{i-1}(x)
<
\frac12-\eta.
\]
Thus
\[
z_i(x)=a_i.
\]
Induction gives
\[
z_i(x)=a_i
\]
for every \(i=1,\ldots,n\) and every \(x\in I_\textbf{a}\). Since \(\textbf{a}\) satisfies \(\Psi\),
\[
d_j(\textbf{a})=0
\]
for every \(j\), and
\[
P(\textbf{a})=0.
\]
Consequently,
\[
G_\Psi(x)=1
\]
for every \(x\in I_\textbf{a}\).

Define
\[
E_{\eta,d}
=
[4\eta,1-4\eta]
\times
\left[\frac14,\frac34\right]^{d-1}.
\]
When \(d=1\), this means
\[
E_{\eta,1}=[4\eta,1-4\eta].
\]
Define
\[
\begin{aligned}
F_{\eta,d}(\mathbf x)
&=
\sigma
\left(
1
-
\frac{1}{4\eta}
\left(
\sigma(4\eta-x_1)
+
\sigma(x_1-1+4\eta)
\right.
\right.
\\
&\left.
\left.
+
\sum_{j=2}^d
\left(
\sigma\left(\frac14-x_j\right)
+
\sigma\left(x_j-\frac34\right)
\right)
\right)
\right).
\end{aligned}
\]
This box-frustum function equals one on \(E_{\eta,d}\). Its support satisfies
\[
\operatorname{supp}F_{\eta,d}
\subset
[0,1]
\times
\left[
\frac14-4\eta,
\frac34+4\eta
\right]^{d-1}
\subset
[0,1]^d.
\]

For nonnegative \(u\) and \(v\), define
\[
\operatorname{mn}(u,v)
=
\sigma
\left(
u-\sigma(u-v)
\right).
\]
Then
\[
\operatorname{mn}(u,v)=\min\{u,v\}.
\]
Set
\[
f_{\Psi,d}(\mathbf x)
=
\operatorname{mn}
\left(
G_\Psi(x_1),
F_{\eta,d}(\mathbf x)
\right).
\]
It follows that
\[
f_{\Psi,d}(\mathbf x)
=
\min
\left\{
G_\Psi(x_1),
F_{\eta,d}(\mathbf x)
\right\}.
\]
In particular,
\[
\operatorname{supp}f_{\Psi,d}
\subset
\operatorname{supp}F_{\eta,d}
\subset
[0,1]^d.
\]

Suppose that \(\Psi\) is unsatisfiable. If
\[
F_{\eta,d}(\mathbf x)>0,
\]
then \(x_1\in[0,1]\), and hence
\[
G_\Psi(x_1)=0.
\]
If
\[
F_{\eta,d}(\mathbf x)=0,
\]
the minimum defining \(f_{\Psi,d}\) is also zero. Therefore
\[
f_{\Psi,d}=0
\]
on \(\mathbb R^d\).

Suppose that \(\Psi\) is satisfiable. Since
\[
2^{-(n+1)}=8\eta,
\]
the definition of \(x_\textbf{a}\) gives
\[
x_\textbf{a}\in[8\eta,1-8\eta].
\]
Consequently,
\[
I_\textbf{a}
\subset
[7\eta,1-7\eta]
\subset
[4\eta,1-4\eta].
\]
Define
\[
R_\textbf{a}
=
I_\textbf{a}
\times
\left[\frac14,\frac34\right]^{d-1}.
\]
When \(d=1\), this means
\[
R_\textbf{a}=I_\textbf{a}.
\]
On \(R_\textbf{a}\),
\[
G_\Psi(x_1)=1
\]
and
\[
F_{\eta,d}(\mathbf x)=1.
\]
Therefore
\[
f_{\Psi,d}=1
\]
on \(R_\textbf{a}\). Its measure is
\[
\mu(R_\textbf{a})
=
2\eta
\left(\frac12\right)^{d-1}
=
\nu_{\eta,d}.
\]
It follows that
\[
\|f_{\Psi,d}\|_{L^p(\mathbb R^d)}^p
\ge
\nu_{\eta,d}
\]
and
\[
\|f_{\Psi,d}\|_{L^p(\mathbb R^d)}
\ge
\nu_{\eta,d}^{1/p}.
\]

Every function in the construction is obtained from rational affine maps and ReLU activations. Previously computed nonnegative quantities can be propagated through later layers by identity ReLU neurons. The sequential extraction network uses polynomially many neurons, and the clause computation and box-frustum construction add polynomially many neurons. The coefficients
\[
\frac{1}{2\eta}=2^{n+3}
\]
and
\[
\frac{1}{4\eta}=2^{n+2}
\]
have polynomial bit length. Thus a rational target network realizing \(f_{\Psi,d}\) can be constructed in polynomial time. Its realization is nonnegative, continuous piecewise affine, globally Lipschitz, and compactly supported.
\end{proof}

\begin{theorem}[Architecture gap]
\label{thm_all_dimensional_architecture_gap}
Fix integers \(d\ge1\) and \(m\ge1\), and fix \(1\le p<\infty\). Given a \(3\)-CNF formula \(\Psi\), one can construct in polynomial time a rational ReLU target network
\[
\mathcal S_{\Psi,d,m}:\mathbb R^d\to\mathbb R^m
\]
and a positive rational tolerance \(\varepsilon_{\Psi,d}\) satisfying
\[
\Psi\notin3\text{-SAT}
\Longrightarrow
\operatorname{OPT}_{p,d,m}
\left(
\mathcal S_{\Psi,d,m},
\varepsilon_{\Psi,d}
\right)
=
0
\]
and
\[
\Psi\in3\text{-SAT}
\Longrightarrow
\operatorname{OPT}_{p,d,m}
\left(
\mathcal S_{\Psi,d,m},
\varepsilon_{\Psi,d}
\right)
\ge
d+2.
\]
Every constructed target is nonzero, componentwise nonnegative, compactly supported, globally Lipschitz, and continuous piecewise affine.

Consequently, the language
\[
\mathsf{LARGE}_{p,d,m}
=
\left\{
(\mathcal S,\varepsilon)
\middle|
\operatorname{OPT}_{p,d,m}(\mathcal S,\varepsilon)
\ge
d+2
\right\}
\]
is NP-hard. Exact two-hidden-layer ReLU neuron minimization is NP-hard for every fixed input dimension \(d\ge1\) and every fixed output dimension \(m\ge1\).
\end{theorem}

\begin{proof}
It is enough to reduce from the restriction of \(3\)-SAT to formulas having at least one variable and one clause. This restriction remains NP-hard.

Let \(\Psi\) have \(n\ge1\) variables. Construct \(f_{\Psi,d}\), \(\eta\), and \(\nu_{\eta,d}\) according to Lemma \ref{lem_sat_target}. Thus
\[
\eta=2^{-(n+4)}
\]
and
\[
\nu_{\eta,d}
=
2\eta
\left(\frac12\right)^{d-1}
=
2^{-n-d-2}.
\]

Define
\[
b_d(\mathbf x)
=
\sigma
\left(
1
-
4
\sum_{j=1}^d
\left(
\sigma(2-x_j)
+
\sigma\left(x_j-\frac94\right)
\right)
\right).
\]
This is the box-frustum function associated with
\[
Q_d
=
\left[2,\frac94\right]^d
\]
and transition parameter \(1/4\). It satisfies
\[
0\le b_d\le1,
\]
equals one on \(Q_d\), and has a transition region where the total violation in its definition lies strictly between \(0\) and \(1/4\). If
\[
b_d(\mathbf x)>0,
\]
then every individual violation in its definition is less than \(1/4\). Consequently,
\[
\operatorname{supp}b_d
\subset
\left[\frac74,\frac52\right]^d.
\]
Therefore
\[
\|b_d\|_{L^p(\mathbb R^d)}^p
\le
\mu
\left(
\left[\frac74,\frac52\right]^d
\right)
=
\left(\frac34\right)^d
\le1,
\]
and hence
\[
\|b_d\|_{L^p(\mathbb R^d)}
\le1.
\]

Set
\[
\varepsilon_{\Psi,d}
=
\frac{\nu_{\eta,d}}{2}
=
2^{-n-d-3}
\]
and
\[
\delta_{\Psi,d}
=
\frac{\nu_{\eta,d}}{4}
=
2^{-n-d-4}.
\]
Define
\[
\mathbf F_{\Psi,d,m}(\mathbf x)
=
\left(
f_{\Psi,d}(\mathbf x)
+
\delta_{\Psi,d}b_d(\mathbf x),
0,
\ldots,
0
\right).
\]
When \(m=1\), the trailing zero coordinates are absent.

Construct \(\mathcal S_{\Psi,d,m}\) by placing the networks for \(f_{\Psi,d}\) and \(b_d\) in parallel. If the two branches have different depths, identity ReLU neurons propagate the nonnegative output of the shorter branch through the remaining layers. The affine output layer returns
\[
f_{\Psi,d}
+
\delta_{\Psi,d}b_d
\]
in its first coordinate and zero in every remaining coordinate.

The network \(\mathcal S_{\Psi,d,m}\) has rational parameters and polynomial size. The numbers
\[
\eta,
\quad
\nu_{\eta,d},
\quad
\varepsilon_{\Psi,d},
\quad
\delta_{\Psi,d}
\]
have polynomial bit length. Since \(d\) and \(m\) are fixed, the baseline network has constant size relative to the size of \(\Psi\).

The support of \(f_{\Psi,d}\) is contained in \([0,1]^d\), while the support of \(b_d\) is contained in \([7/4,5/2]^d\). Thus the two supports are disjoint. In particular,
\[
f_{\Psi,d}=0
\]
and
\[
b_d=1
\]
on \(Q_d\). Hence
\[
\mathbf F_{\Psi,d,m}
=
\delta_{\Psi,d}\mathbf e_1
\]
on \(Q_d\), where
\[
\mathbf e_1=(1,0,\ldots,0)\in\mathbb R^m
\]
is the first standard basis vector. Since \(\delta_{\Psi,d}>0\), every constructed target is nonzero. The target is componentwise nonnegative, compactly supported, globally Lipschitz, and continuous piecewise affine.

Suppose that
\[
\Psi\notin3\text{-SAT}.
\]
Lemma \ref{lem_sat_target} gives
\[
f_{\Psi,d}=0.
\]
Therefore
\[
\mathbf F_{\Psi,d,m}
=
\delta_{\Psi,d}b_d\mathbf e_1.
\]
Using
\[
\|b_d\|_{L^p(\mathbb R^d)}
\le1,
\]
we obtain
\[
\|\mathbf F_{\Psi,d,m}\|_{L^p(\mathbb R^d,\mathbb R^m)}
\le
\delta_{\Psi,d}
=
\frac{\nu_{\eta,d}}4
<
\frac{\nu_{\eta,d}}2
=
\varepsilon_{\Psi,d}.
\]
Since the approximation error of the zero network satisfies
\[
\|0-\mathbf F_{\Psi,d,m}\|_{L^p(\mathbb R^d,\mathbb R^m)}
=
\|\mathbf F_{\Psi,d,m}\|_{L^p(\mathbb R^d,\mathbb R^m)}
<
\varepsilon_{\Psi,d},
\]
the zero network is feasible. Because the hidden-neuron count cannot be negative,
\[
\operatorname{OPT}_{p,d,m}
\left(
\mathcal S_{\Psi,d,m},
\varepsilon_{\Psi,d}
\right)
=
0.
\]
Suppose that
\[
\Psi\in3\text{-SAT}.
\]
Let \(R_\textbf{a}\) be the box supplied by Lemma \ref{lem_sat_target}. Then
\[
\mu(R_\textbf{a})=\nu_{\eta,d}
\]
and
\[
f_{\Psi,d}=1
\]
on \(R_\textbf{a}\). Since
\[
R_\textbf{a}\subset[0,1]^d
\]
and
\[
\operatorname{supp}b_d
\subset
\left[\frac74,\frac52\right]^d,
\]
one has
\[
b_d=0
\]
on \(R_a\). Therefore
\[
\mathbf F_{\Psi,d,m}
=
\mathbf e_1
\]
on \(R_\textbf{a}\).

It follows that
\[
\|\mathbf F_{\Psi,d,m}\|_{L^p(\mathbb R^d,\mathbb R^m)}^p
\ge
\nu_{\eta,d}.
\]
Since
\[
0<\nu_{\eta,d}\le1
\]
and
\[
p\ge1,
\]
we have
\[
\|\mathbf F_{\Psi,d,m}\|_{L^p(\mathbb R^d,\mathbb R^m)}
\ge
\nu_{\eta,d}^{1/p}
\ge
\nu_{\eta,d}
=
2\varepsilon_{\Psi,d}.
\]
Thus the zero network is not feasible and is therefore excluded from the set of feasible approximating networks.

Let \(\Phi\) be any feasible two-hidden-layer ReLU network. Feasibility gives
\[
\Phi-\mathbf F_{\Psi,d,m}
\in
L^p(\mathbb R^d,\mathbb R^m).
\]
Since
\[
\mathbf F_{\Psi,d,m}
\in
L^p(\mathbb R^d,\mathbb R^m),
\]
the triangle inequality gives
\[
\Phi\in L^p(\mathbb R^d,\mathbb R^m).
\]
The function \(\Phi\) is nonzero because the zero function is not feasible. Lemma \ref{lem_minimum_nonzero_Lp_neurons} shows that \(\Phi\) has at least \(d+2\) hidden neurons. Therefore
\[
\operatorname{OPT}_{p,d,m}
\left(
\mathcal S_{\Psi,d,m},
\varepsilon_{\Psi,d}
\right)
\ge
d+2.
\]

The optimum is finite. Indeed, every component of \(\mathbf F_{\Psi,d,m}\) belongs to \(L^p(\mathbb R^d)\). Corollary \ref{cor_vector_frustum_density} gives a two-hidden-layer ReLU network whose approximation error is less than \(\varepsilon_{\Psi,d}\).

The construction of
\[
\mathcal S_{\Psi,d,m}
\]
and
\[
\varepsilon_{\Psi,d}
\]
is polynomial in the size of \(\Psi\). We have proved
\[
\Psi\in3\text{-SAT}
\Longleftrightarrow
\left(
\mathcal S_{\Psi,d,m},
\varepsilon_{\Psi,d}
\right)
\in
\mathsf{LARGE}_{p,d,m}.
\]
This is a polynomial-time many-one reduction from \(3\)-SAT. Therefore
\[
\mathsf{LARGE}_{p,d,m}
\]
is NP-hard. An exact minimization algorithm would decide this language by comparing the returned optimum with \(d+2\).
\end{proof}

\begin{theorem}[NP-hardness of exact neuron minimization]
\label{thm_exact_neuron_minimization}
Fix integers \(d\ge1\) and \(m\ge1\), and fix \(1\le p<\infty\). Computing
\[
\operatorname{OPT}_{p,d,m}(\mathcal S,\varepsilon)
\]
exactly is NP-hard.
\end{theorem}

\begin{proof}
Given a \(3\)-CNF formula \(\Psi\), Theorem \ref{thm_all_dimensional_architecture_gap} constructs
\[
\left(\mathcal S_{\Psi,d,m},\varepsilon_{\Psi,d}\right)
\]
in polynomial time and guarantees
\[
\Psi\in3\text{-SAT}
\Longleftrightarrow
\operatorname{OPT}_{p,d,m}
\left(
\mathcal S_{\Psi,d,m},
\varepsilon_{\Psi,d}
\right)
\ge d+2.
\]
Thus, if the optimum could be computed exactly in polynomial time, then \(3\)-SAT could be decided in polynomial time by checking whether the optimum is at least \(d+2\). Hence exact two-hidden-layer ReLU neuron minimization under the \(L^p\) approximation constraint is NP-hard.
\end{proof}

\section{Conclusion}

This paper establishes that minimizing the total number of hidden neurons in two-hidden-layer ReLU networks under an \(L^p\) approximation constraint is NP-hard for every fixed input dimension \(d\ge1\), output dimension \(m\ge1\), and \(1\le p<\infty\). The proof gives a polynomial-time reduction from \(3\)-SAT using rational ReLU target networks whose realizations are nonzero, componentwise nonnegative, compactly supported, globally Lipschitz, and continuous piecewise affine. For an unsatisfiable formula, the zero network satisfies the approximation constraint and the optimum is zero. For a satisfiable formula, every feasible two-hidden-layer ReLU network has at least \(d+2\) hidden neurons. This gap establishes a worst-case computational barrier to exact architecture selection and motivates the use of heuristic, approximate, or structurally restricted methods for ReLU network design.

\bibliographystyle{IEEEtran}
\bibliography{references}

\end{document}